\pdfoutput=1

\documentclass[11pt]{article}

\usepackage{acl}

\usepackage{times}
\usepackage{latexsym}

\usepackage[T1]{fontenc}

\usepackage[utf8]{inputenc}

\usepackage{microtype}

\usepackage{inconsolata}

\usepackage{graphicx}
\usepackage{makecell}
\usepackage{amsmath}
\usepackage{amsfonts}
\usepackage{booktabs}
\usepackage{multirow}

\usepackage{tabularx}
\usepackage{arydshln}

\usepackage{enumitem}
\usepackage{float} 
\usepackage{amssymb}

\newcommand{\bW}{\mathbf{W}}
\newtheorem{prop}{Proposition} 
\newtheorem{lemma}{Lemma} 
\newtheorem{theorem}{Theorem}

\title{CAGE: Coherence-Aware Graph Encoding for Retrieval-Augmented Generation}

\author{
\textbf{Tong Qi\textsuperscript{1}},
\textbf{Jingyu Wu\textsuperscript{1}},
\textbf{Youbing Yin\textsuperscript{1}},
\textbf{Spencer Hong\textsuperscript{2}},
\textbf{Daben Liu\textsuperscript{1}},
\textbf{Erin Babinsky\textsuperscript{1}} \\
\texttt{\{tong.qi, jingyu.wu, youbing.yin, daben.liu, erin.babinsky\}@capitalone.com\textsuperscript{1}} \\
\texttt{spencer@generalintelligencecompany.com\textsuperscript{2}}
}

\begin{document}
\maketitle
\begin{abstract}

Traditional Retrieval-Augmented Generation (RAG) systems score each passage independently against the query, assembling context sets that may be individually relevant yet collectively incoherent. We introduce \textbf{Coherence-Aware Graph Encoding (CAGE)}, a reranking framework that models "between-chunk coherence" across four dimensions: \textbf{Intra-Domain Relevance, Noise Resistance, Informational Bonding, and Factual Consistency}. Our pipeline transforms retrieved passages into directed heterogeneous entity graphs, amplifies factual anchors via min-out-degree reweighting, encodes structural patterns through a Relational Graph Convolutional Network, and fuses inter-chunk coherence with query relevance for final ranking. Evaluated across four multi-hop benchmarks, CAGE matches or outperforms strong baselines including monoT5 in Recall@5 on bridge-dominated datasets and consistently improves downstream Exact Match, demonstrating that structurally coherent context yields more precise answers even when retrieval recall is comparable or lower.

\end{abstract}

\section{Introduction}
Retrieval-Augmented Generation (RAG) has become a standard paradigm for grounding large
language model (LLM) outputs in external knowledge. In the canonical setup, a retriever
identifies the top-k passages most relevant to a query, and those passages are concatenated
as context for the LLM. The quality of this retrieved set therefore acts as a hard ceiling on
downstream answer quality: no matter how capable the generator, it cannot synthesize a
correct answer from an incoherent or contradictory evidence base.

Current retrieval pipelines address relevance through two stages: an initial retriever that ranks candidates by query chunk similarity, followed by a reranker that refines those scores using more expressive query–chunk interaction models \cite{nogueira2020document,santhanam2022colbertv2}. Despite this increased expressiveness, both stages share a common objective: scoring each chunk individually against the query. Neither explicitly models the structural dependencies among the retrieved chunks themselves. As a result, the top-k set is assembled as a collection of independently high-scoring fragments rather than as a collectively coherent unit of evidence. Consider a query such as \textit{"What is the current treatment procedure for Type 2 diabetes?"} A retrieval pipeline may promote: (1) a 2010 passage recommending early insulin therapy; (2) a 2023 passage on GLP-1 receptor agonists; (3) a passage on Type 1 diabetes as a hereditary condition; and (4) a passage disputing the medical benefit of GLP-1 treatments. Each chunk scores well against the query individually; together, they form a contradictory and temporally fragmented context that actively misleads the generator. This failure mode is not an edge case; in our experiments across four multi-hop benchmarks, we observe that gold supporting chunks are frequently displaced by hard distractors that outscore them on individual query-relevance metrics (see section \ref{subsec:retrievalperformance}).

We distinguish this problem sharply from related notions of RAG quality. Existing evaluation frameworks such as RAGAS  \cite{es2024ragas} and TruLens \cite{trulens} measure generation-side faithfulness and hallucination, properties of the produced answer relative to the retrieved context. Graph-based RAG systems such as GraphRAG \cite{edge2024local} and GNN-RAG \cite{mavromatis2025gnn} use graph structures for information navigation, treating the graph as a retrieval map rather than a coherence validator. 
Our work addresses a different, earlier-stage problem: between-chunk coherence, defined as the structural and logical integrity of the retrieved chunk set prior to answer generation. A retrieved set can exhibit high per-chunk relevance scores and produce a fluent generated answer, yet still be incoherent, if the chunks share keyword overlap but contradict one another on the target facts, or fail to form a connected reasoning chain, because the generator's precision is compromised regardless of its capability.

We argue that between-chunk coherence is not a monolithic property. In Section \ref{sec:def} we formalize a decomposition into four distinct dimensions, Intra-domain Relevance (entity-level topical alignment beyond surface keyword overlap), Noise Resistance (suppression of distractor chunks that share query keywords but lack inferential utility), Informational Bonding (structural connectivity enabling multi-hop traversal through bridge entities), and Factual Consistency (logical compatibility across the relational assertions of retrieved chunks). Together, these four dimensions define what it means for a retrieved set to constitute a coherent evidence base, and they motivate each component of the framework we introduce.

To capture these dimensions jointly, we propose Coherence-Aware Graph Encoding (CAGE), a graph-based reranking framework that transforms retrieved chunks into directed heterogeneous entity graphs, encodes them via Relational Graph Convolutional Networks (R-GCN; \citet{schlichtkrull2018modeling}), and scores each chunk using a fused objective that combines query relevance with inter-chunk structural coherence. Graph reweighting via min-out-degree node emphasis amplifies the contribution of specific, semantically grounded entities (the natural anchors of factual claims) while the R-GCN's relation-type-specific weight matrices allow the model to penalize structural contradictions and reward inferential bridging across documents.
Evaluated across four multi-hop benchmarks, CAGE outperforms strong retrieval baselines in Recall@5 and consistently achieves superior Exact Match performance in downstream generation, demonstrating that promoting structurally coherent evidence translates directly into more precise answers.

\section{Related Work}
\subsection{Coherence in Information Retrieval}
Text coherence has traditionally been studied as an intrinsic property of single documents — the logical connectivity between constituent sentences. Foundational work used sentence ordering and document discrimination tasks \citet{liu2020evaluating,mesgar2016lexical}, while graph-based approaches tracked entity transitions across adjacent sentences to compute local coherence \citet{guinaudeau2013graph,mesgar2014normalized, kramov2020evaluating}. These models share a key insight: coherence is relational, emerging from structural dependencies rather than from sentences in isolation.

However, they assume continuous, single-author discourse. The retrieval setting violates this: passages come from independent sources with no shared discourse plan. Current RAG evaluation frameworks sidestep this challenge entirely. RAGAS \citet{es2024ragas} and TruLens \citet{trulens} measure generation-side properties: faithfulness, hallucination, and coherence, where coherence evaluates the logical consistency and flow of the AI-generated response, not the structural integrity of the retrieved chunks that produced it. NLI-based contradiction detectors \cite{huang2025survey,ming2024faitheval} operate at the claim level but miss higher-order structure: whether passages form connected reasoning chains or collectively support a single answer. We formalize between-chunk coherence, the structural and logical integrity of a retrieved set, as a distinct axis orthogonal to both per-chunk relevance and downstream faithfulness.

\subsection{Coherence Gap in Retrieval Augmented Generation}
We first define the baseline our work builds upon. Vanilla RAG denotes the standard retrieve-then-generate pipeline: all-MiniLM-L6-v2 \cite{reimers2019sentence} encodes queries and passages into a shared embedding space, retrieves the top-k by cosine similarity, and concatenates them as context for the generator. No inter-passage signal is considered. CAGE applies graph reranking to this same retrieved set, isolating the contribution of coherence-aware scoring.

Recent work has moved beyond flat retrieval toward graph-based paradigms. GraphRAG \cite{edge2024local} builds hierarchical knowledge graphs for query-focused summarization. HippoRAG \cite{gutierrez2024hipporag} enables associative multi-hop retrieval via Personalized PageRank. GNN-RAG \cite{mavromatis2025gnn} trains GNNs over knowledge bases to identify reasoning paths. HyperGraphRAG \cite{luo2025hypergraphrag} introduces hypergraph representations capturing many-to-many entity–document relationships. G-RAG \cite{dong2024don} proposes graph-based reranking to leverage structural connectivity among candidate passages.

Despite their sophistication, these systems share a common philosophy: the graph serves as a navigation map for locating relevant passages, not as a validator of coherence among them. GraphRAG improves coverage, but does not assess whether the retrieved passages are logically coherent. HippoRAG increases multi-hop recall without guaranteeing internal consistency. GNN-RAG guides which nodes to visit, yet assembles the final set without a coherence objective. HyperGraphRAG captures group associations but remains agnostic to contradictions within retrieved members. 
While G-RAG models local passage connections during reranking, its approach evaluates pairwise graph topological ties rather than validating set-level coherence evaluation. 
In each case, graph topology is exploited for retrieval reach or pairwise affinity rather than set-level evidence integrity.

This leaves a persistent gap in retrieval-augmented generation: existing graph-based reranking and set-selection paradigms remain largely coherence-agnostic at the set level. Because prior methods evaluate candidates independently or rely on localized topological affinity, nothing in their design prevents the final context set from harboring mutual contradictions, lacking inferential bridges, or accumulating uninformative distractors.
To bridge this gap, we introduce CAGE, a post-retrieval reranking framework operating prior to answer generation. Rather than utilizing graph representations merely for information navigation, our primary contribution is establishing between-chunk coherence as an explicit optimization objective for scoring retrieved candidate sets. By repurposing graph topology into a mechanism for set-level evidence validation, CAGE ensures that concatenated contexts form a structurally integrable foundation before answer synthesis.

\begin{figure*}[h!]
    \centering 
    \includegraphics[width=\textwidth]{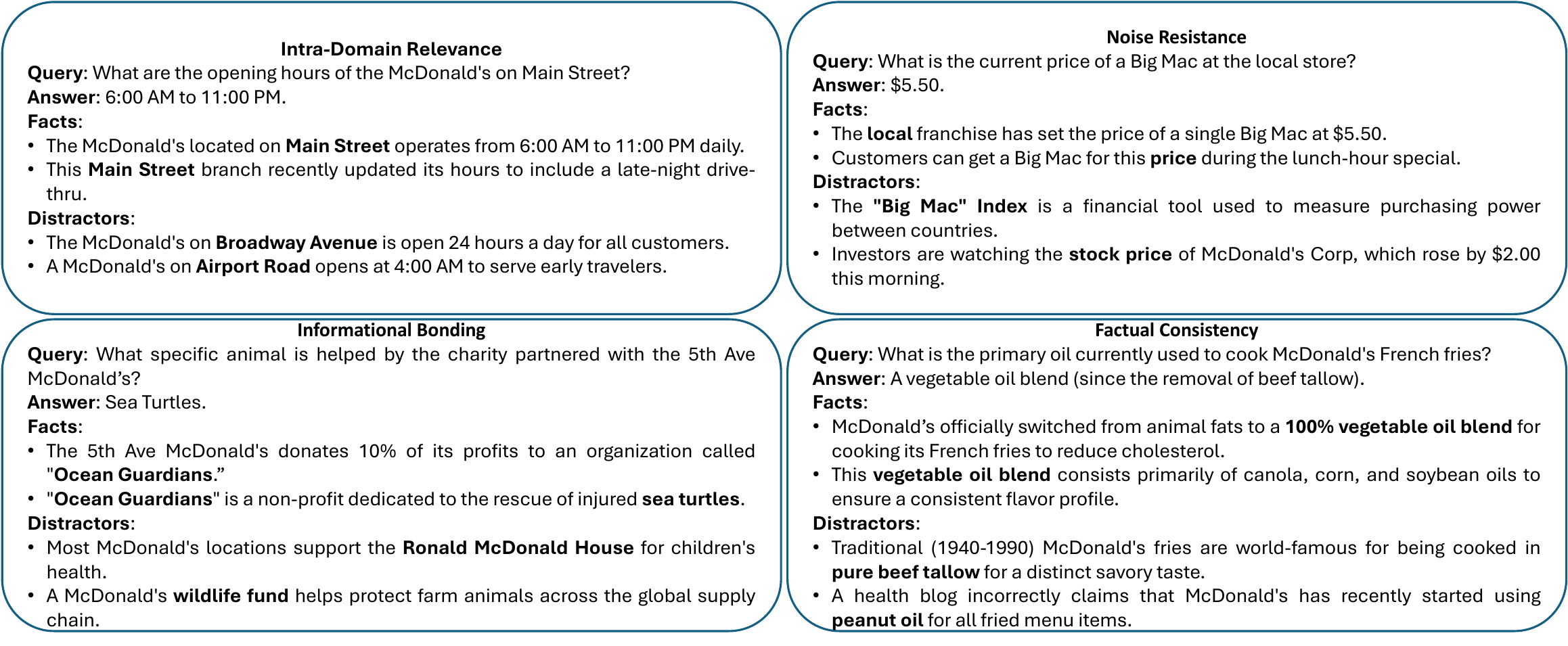}
    \caption{Examples of the four coherence dimensions. Intra-domain relevance (top-left): distractors share the entity type (McDonald's) but anchor to a different referent (Broadway vs. Main Street). Noise resistance (top-right): cross-domain keyword overlap ("Big Mac" as food vs. financial index) produces topologically disconnected distractors. Informational bonding (bottom-left): a bridge entity ("Ocean Guardians") is required to connect disjoint facts into a traversable reasoning chain. Factual consistency (bottom-right): mutually exclusive claims (vegetable oil vs. beef tallow) violate logical integrability of the retrieved set.}
    \label{fig:dimexample}
\end{figure*}

\section{Background Definition and Theory}
\label{sec:def}
The preceding sections identify two gaps: (1) coherence remains formally undefined for sets of independently retrieved passages, and (2) no existing graph-based RAG system exploits relational structure as a coherence signal. We address both by decomposing between-chunk coherence into four measurable dimensions and proving that Relational Graph Convolutional Networks are expressively sufficient to distinguish coherent evidence sets from incoherent ones.

\subsection{Multiple-Dimensions of Coherence in Retrieval}
We formalize between-chunk coherence as a composite property over a retrieved set $({C_1, \ldots, C_k}) \in \mathcal{D}$ with respect to query $q$. We decompose it into four dimensions and the example is shown in Figure \ref{fig:dimexample}:
\begin{itemize}[leftmargin=*]
    \vspace{-0.2cm}\item \textit{\textbf{Intra-Domain Relevance}}: Entity-level referential alignment among co-retrieved passages within the same domain. Standard retrievers may surface chunks that share domain vocabulary with the query yet anchor to different real-world referents (e.g., two McDonald's branches on different streets). Intra-domain relevance requires that passages ground to a shared referential context, not merely a shared topic. 
    \vspace{-0.2cm}\item \textit{\textbf{Noise Resistance}}: Suppression of cross-domain distractor chunks that exhibit high query similarity but contribute no inferential utility. These arise from keyword overlap across unrelated domains (e.g., "Big Mac" as a food item vs. a financial index) and act as isolated informational islands, topologically disconnected from the evidence backbone.
    \vspace{-0.2cm}\item \textit{\textbf{Informational Bonding}}: The presence of bridge entities enabling multi-hop traversal across passages. A coherent set must contain a structurally continuous path from query entities to the target answer; without such bridges, the set remains disconnected facts rather than a traversable reasoning chain.
    \vspace{-0.2cm}\item \textit{\textbf{Factual Consistency}}: Logical compatibility among the relational assertions of co-retrieved passages. Multi-source retrieval can introduce mutually exclusive claims (e.g., conflicting dates, contradictory causal relations). A coherent set must be logically integrable, its facts must not invalidate one another during synthesis.
\end{itemize}

These dimensions are complementary and non-overlapping: intra-domain relevance and noise resistance partition the relevance space by scope (within-domain disambiguation vs. cross-domain filtering), while informational bonding and factual consistency address structural connectivity and logical harmony respectively. Qualitative examples appear in Appendix \ref{app:example} (Figures \ref{fig:dim1example}-\ref{fig:dim4example}).

\subsection{Theoretical Foundation of Relational Coherence}

We establish that R-GCNs are expressively sufficient to capture all four coherence dimensions defined above. Specifically, we show (1) that the R-GCN update rule is injective over relational multisets, matching the discriminative power of the Relational Weisfeiler-Lehman algorithm, and (2) that each coherence dimension corresponds to a distinguishable structural property in the resulting embedding space. See Appendix B.\ref{lemma1}-B.\ref{lemma4} for Propositions/Lemmas 1-4 individually.

\textbf{Notation:} $G = (V, E, R)$: multi-relational graph with nodes $V$, edges $E$, relation types $R$. $h_i^{(l)}$: embedding of node $i$ at layer $l$ with total $L$ layers. $\bW_r^{(l)}$: relation-specific weight matrix. $\Phi: G \to \mathcal{E}$: graph encoding function. $\sigma$: non-linear activation.

\begin{prop}{Relational WL-Equivalence}
\label{prop}
Let $G = (V, E, R)$ be a multi-relational graph, and consider the unnormalized R-GCN update
  \[
  h_i^{(l+1)} = \sigma \left( \sum_{r \in \mathcal{R}} \bW_r^{(l)} \sum_{j \in \mathcal{N}_i^r} h_j^{(l)} + \bW_0^{(l)} h_i^{(l)} \right).
  \]
An $L$-layer R-GCN under this update, with relation-specific weight matrices $\{\bW_r^{(l)}\}$ of sufficient rank, is as discriminative as the $L$-step Relational Weisfeiler-Lehman algorithm: the update is injective over relational multisets, so $h_u^{(L)} = h_v^{(L)}$ if and only if nodes $u$ and $v$ have isomorphic $L$-hop relational neighborhoods. This guarantees that, under sum aggregation, structurally distinct chunk graphs, differing in entity relations, connectivity, or directionality, receive distinct embeddings, a prerequisite for any coherence scoring function that must discriminate coherent sets from incoherent ones.
\end{prop}

Building upon Proposition 1, we posit that the graph embedding manifold acts as a structural validator for retrieval coherence.
\begin{theorem}{Coherence Separability} 
\label{thm1}
Let $\Phi: \mathcal{G} \to \mathbb{R}^d$ be an $L$-layer R-GCN with injective aggregation (Proposition~\ref{prop}, exact under unnormalized sum aggregation and approximate under our degree-normalized implementation per the condition in Appendix~B.1) followed by global sum pooling: $\Phi(G) = \sum_{i \in V} h_i^{(L)}$. Let $S(\Phi(G_i), \Phi(G_j)) = \cos(\Phi(G_i), \Phi(G_j))$. If two passage graphs $G_i, G_j$ share entities and compatible relation types -- either directly, or transitively through a bridge entity relationally enriched within each graph's own 1-hop neighborhood (Lemma \ref{lemma3}) and neither contains contradictory relational patterns relative to the other (i.e., exhibit structural coherence), then:
    \[
    S(\Phi(G_i), \Phi(G_j)) > S(\Phi(G_i), \Phi(G_k))
    \]
for any $G_k$ that shares no such entities or bridge structure with $G_i$, or contains contradictory relational patterns.
\end{theorem}

\textbf{Proof Sketch} 
By Proposition \ref{prop}, structurally distinct graphs receive distinct embeddings. We decompose the structural conditions under which similarity increases: shared linguistic features produce non-orthogonal initial projections (Lemma \ref{lemma1}), isolated nodes are attenuated during aggregation (Lemma \ref{lemma2}), shared bridge entities receive independent relational enrichment in each passage's own graph, elevating cross-graph similarity beyond raw lexical overlap (Lemma~\ref{lemma3}), and relational contradictions produce divergent embeddings via non-commutativity of $\bW_r$ matrices (Lemma~\ref{lemma4}). Together, these establish that $S$ separates coherent passage pairs from incoherent ones. While Theorem~\ref{thm1} establishes a general $L$-layer formulation to provide a scale-invariant framework for arbitrary graph diameters, we empirically instantiate CAGE with $L=1$. For our compact passage graphs, bridge entities' relational context is already captured within each passage's own immediate 1-hop neighborhood (Lemma~\ref{lemma3}), so a single layer combined with global sum pooling suffices to elevate cross-graph similarity without requiring deeper propagation or multi-layer over-smoothing. The general $L$-layer theory thus serves as a mathematical blueprint for extending CAGE to long-form contexts or macro-reasoning tasks requiring deeper propagation.

Our proof proceeds by showing that the Relational GCN mapping preserves the structural distinctness of relational multisets. By identifying these distinct neighborhoods as anchors for a coherence manifold, we establish the existence of an injective estimator $\mathcal{S}$. See Appendix \ref{sec:appx_proof} for the full derivation.

By projecting documents into a relational manifold rather than a static semantic space, our method treats retrieval as a validation of structural isomorphism. This framework ensures that retrieved evidence is not merely "topically related" but is topologically and logically integrable for downstream reasoning.

\section{CAGE}

Our pipeline takes the top-$k$ passages from the initial retriever and produces a coherence-aware reranking in four stages: graph construction from each passage, min-out-degree node reweighting, R-GCN encoding, and fused scoring that combines query relevance with inter-chunk coherence. Figure \ref{fig:flow} illustrates the full pipeline.

\begin{figure*}[t!]
    \centering 
    \small
    \includegraphics[width=\textwidth]{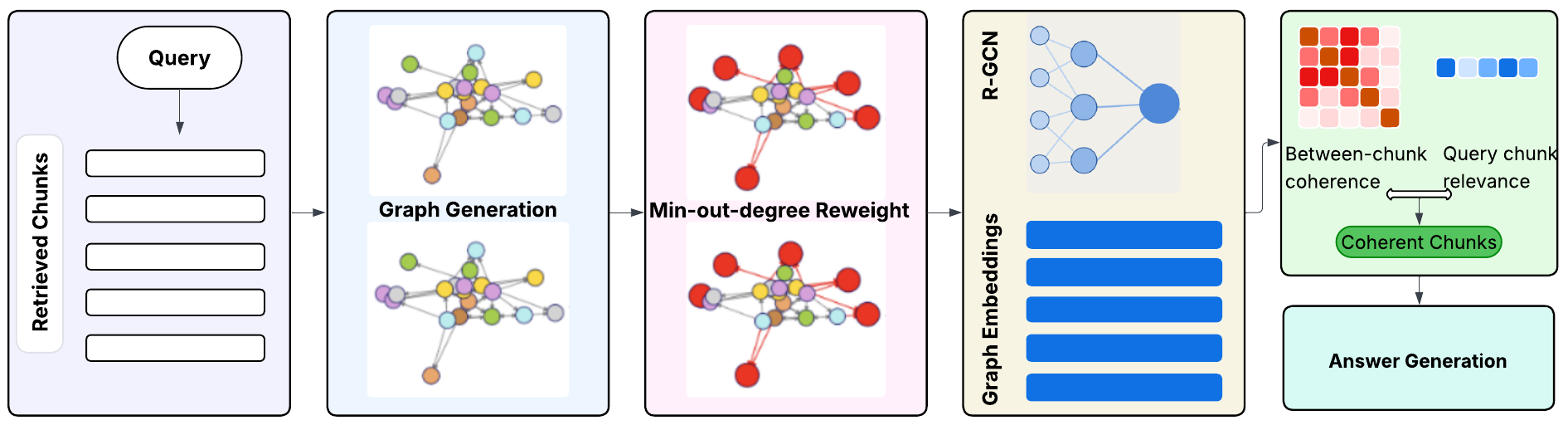}
    \caption{Overview of the CAGE pipeline. Retrieved passages are converted into directed heterogeneous entity graphs, reweighted by min-out-degree to emphasize factual anchors (highlighted nodes), encoded via R-GCN into graph-level embeddings, and scored using a fused objective combining query relevance with pairwise structural coherence.}
    \label{fig:flow}
\end{figure*}


\subsection{Entity Enhanced Graph Construction}

For each document $C_i$, we apply spaCy's dependency parser (en\_core\_web\_lg) with coreference resolution to extract a directed heterogeneous graph $G_i = (V_i, E_i)$. Nodes represent lemmatized content words with pronouns resolved to their canonical referents. Edges encode various relation types derived from linguistic analysis, such as subject–object links labeled by their governing verb lemma (with negation), adjectival/compound modifiers and appositives, numeric and quantity associations, prepositional relations (e.g., spatial, temporal), and named-entity co-occurrence within sentence boundaries. Each node is initialized with a feature vector consisting of its 300-dimensional contextual embedding (L2-normalized), one-hot POS tag, one-hot entity type, and a scalar weight. 

\subsection{Graph Reweighting}
To emphasize coherence-relevant parts of the graph, we apply a reweighting strategy focused on min-out-degree nodes:
$$
V_{i, \text{min}} = \{ v \in V_i \mid \text{deg}^+(v) = \delta^+(G_i) \} 
$$
where $\text{deg}^+(v)$ is the out-degree of node v, and $\delta^+(G_i)$ denotes the minimum out-degree over all nodes in graph $G_i$. These nodes often represent terminal or specific concepts, such as named entities, numeric values, or endpoints of factual statements. We scale their feature vectors and adjacent edge weights by a factor $\beta > 1$. Their low out-degree suggests they are more semantically grounded, playing key roles in anchoring the meaning of the chunk. This structural prior helps the model prioritize information from nodes likely to correspond to specific entities, facts, or low-level concepts, improving downstream graph representation learning.

\subsection{R-GCN Encoding}
\label{sec:4.3}
We encode each reweighted graph using a single R-GCN layer \cite{schlichtkrull2018modeling} with ReLU activation, followed by global sum pooling to obtain a graph-level embedding $Z_i \in \mathbb{R}^{128}$. To prevent representation collapse, the encoder is trained with a diversity objective that minimizes mean pairwise cosine similarity within each batch. 
This objective acts as a regularization mechanism over the embedding space, forcing structurally distinct passage graphs to occupy separated regions so that pairwise distance metrics remain discriminative at test time. 
Consequently, high downstream similarity scores specifically reflect genuine structural overlap, such as shared entities, compatible relations, and connected reasoning paths, rather than uninformative, collapsed representations. We interpret this relational alignment as between-chunk coherence.
\subsection{Coherence Scoring and Reranking} Given graph embeddings ${Z_1, \ldots, Z_k}$, we compute a symmetric coherence matrix $C \in \mathbb{R}^{k \times k}$:
$$C_{ij} = \cos(Z_i, Z_j)$$
The coherence score for passage $i$ aggregates its structural alignment with all other retrieved passages:
$$s_i = \sum_{j=1}^{k} C_{ij} \quad \text{for } j = 1, \dots, k \text{   and  }  i\neq j.$$
Passages with high $s_i$ share entities, compatible relations, and connected paths with the rest of the retrieved set; passages with low $s_i$ are structural outliers likely representing noise or distractors. To balance coherence with query relevance, the final ranking score fuses both signals:
$$f_i = \alpha \cdot \text{sim}(q, C_i) + (1 - \alpha) \cdot \frac{1}{k-1} s_i$$
where $\text{sim}(q, C_i)$ is the query–passage cosine similarity from the initial retriever, and $\alpha \in [0, 1]$ controls the trade-off between relevance and coherence. The normalization by $k-1$ ensures the coherence term is scale-comparable to the query similarity. Passages are reranked by $f_i$ and the top-$k$ are passed to the generator.


\begin{table*}[!ht]
\centering
\small
\begin{tabular}{lccccc}
\toprule
\textbf{Retriever} & \textbf{MuSiQue} & \textbf{2Wiki} & \textbf{HotpotQA} & \textbf{RAMDocs}  \\
\midrule
BM25\cite{robertson2009probabilistic}              & 49.6          & 71.6          & 76.4          &  86.8\\

GTR\cite{ni2022large}               & 49.1     & 67.9   & 73.3     &  \underline{87.4} \\
ColBERTv2\cite{santhanam2022colbertv2}        & 49.2     & 68.2    & 79.3     & \textbf{87.6}  \\
RAPTOR\cite{sarthi2024raptor}        & 45.3          & 53.8          & 71.2     &  87.4  \\
HyperGraphRAG\cite{luo2025hypergraphrag} &     36.8      &      45.9    & 60.5       &  87.2  \\
mxbai-rerank\cite{li2025prorank} & 63.5  & \underline{84.0}&  83.9&  87.1\\
monoT5\cite{pradeep2021expando} & 62.8 & 78.1 & \underline{84.2} & 87.1  \\
Vanilla RAG   &   \underline{63.6}       &      80.4     & 83.8          &  87.2  \\
CAGE              & \textbf{64.0}        & \textbf{85.4} & \textbf{84.2}         & 87.0 \\

\bottomrule
\end{tabular}
\caption{Recall@5 across four multi-hop benchmarks. CAGE achieves the highest Recall@5 on 2WikiMultihopQA and HotpotQA, matching or outperforming the strong baselines, while remaining competitive on MuSiQue and RAMDocs.}
\label{tab:recalls_r5}
\end{table*}


\section{Experimental Setup}

\subsection{Dataset}
We benchmark across four multi-hop datasets: HotpotQA (distractor setting) \citep{yang2018hotpotqa}, 2WikiMultihopQA\citep{ho2020constructing}, MuSiQue \citep{trivedi2022musique}, and RAMDocs \citep{wang2025retrieval}. 
Because isolated per-dimension annotations do not yet exist in standard RAG benchmarks, we therefore leverage multi-hop gold supporting-fact chains as a principled empirical proxy. Gold supporting chains represent the unique evidence subset that must collectively exhibit structural connectivity and logical harmony to form a valid reasoning path, whereas hard distractors act as semantically similar yet topologically disconnected alternatives. 

The four datasets collectively stress-test different coherence dimensions (Table~\ref{tab:datasets_summary}, Appendix \ref{app:data}). HotpotQA (81\% bridge, 19\% comparison) and 2WikiMultihopQA (48\% bridge, 41\% comparison, 11\% inference) emphasize \textit{informational bonding} and \textit{intra-domain relevance} via multi-hop bridge questions that require disambiguating entities from the same domain. MuSiQue introduces 3-4 hop chains with high distractor density, testing \textit{noise resistance} and \textit{intra-domain relevance}, its distractors are intentionally engineered with shared professional vocabulary to be highly adversarial yet structurally disconnected. RAMDocs targets \textit{factual consistency} via adversarial misinformation that modifies only the target assertion while preserving surrounding entity structure. Qualitative case studies mapping each dataset to its primary coherence dimension appear in Appendix \ref{app:example} (Figures \ref{fig:dim1example}–\ref{fig:dim4example}).

\subsection{Setup}  

All methods retrieve from the same candidate pool per dataset: the full set of passages provided by each benchmark (including gold supporting facts and distractors). For each query, we construct entity graphs from the top-$k$ retrieved passages using spaCy (en\_core\_web\_lg) with coreference resolution, apply min-out-degree reweighting ($\beta = 2$), and train a single-layer R-GCN independently on that query's $k$ passage graphs. The R-GCN takes node features (300-dim contextual embedding, one-hot POS, one-hot entity type, weight scalar), edge indices, and edge types as input, producing 128-dimensional graph-level embeddings via global sum pooling. Training uses a diversity objective that minimizes mean pairwise cosine similarity across the $k$ graph embeddings, encouraging structurally distinct passages to occupy distinct regions of the embedding space. Since each training instance consists of only $k$ small graphs (typically 10–50 nodes each), training converges within 30–60 epochs in under a second per query using Adam (lr = 1e-3). No offline pre-training or corpus-level computation is required. The fusion weight is fixed at $\alpha =0.7$, which is selected via grid search over $\alpha \in ({0.1, 0.2, \ldots, 1.0})$ and fixed across all datasets as Figure \ref{fig:alpha} shows in Appendix \ref{app:figs}.


\section{Results}
\subsection{Retrieval Performance}
\label{subsec:retrievalperformance}

We evaluate whether our fused scoring successfully promotes gold supporting chunks into the top-5 positions. Baselines span retrieval paradigms: BM25 \cite{robertson2009probabilistic} for sparse retrieval and GTR \cite{ni2022large} for dense bi-encoder retrieval, ColBERTv2 \cite{santhanam2022colbertv2} for late interaction, RAPTOR \cite{sarthi2024raptor} for recursive abstractive retrieval, and HyperGraphRAG \cite{luo2025hypergraphrag} for hypergraph-structured retrieval, monoT5 \cite{pradeep2021expando} and  mxbai-rerank \cite{li2025prorank} for cross-encoder reranking. Vanilla RAG (Section 2.2) serves as our ablation baseline; it shares the same first-stage retriever as CAGE but applies no graph reranking.

We report Recall@5: the fraction of gold supporting chunks that appear in the top-5 after ranking. This metric is particularly revealing for coherence-aware reranking because bridge passages, structurally necessary for the reasoning chain but semantically distant from the query, often fail to surface under relevance-only ranking.

Figure \ref{fig:heatmapexample1} visualizes this effect on a single example. The left heatmaps show pairwise semantic similarity and query relevance: the second gold passage scores only 0.29 against the query and exhibits weak pairwise similarity to the other two gold passages. Under relevance-only ranking, this passage would be displaced by higher-scoring distractors. The right heatmap shows our structural coherence scores: all three gold passages converge into a distinct high-similarity cluster, cleanly separated from distractors. The R-GCN recovers the latent connectivity, shared entities and compatible relations, that surface similarity misses, effectively rescuing a semantically distant but structurally essential bridge passage.

Table \ref{tab:recalls_r5} presents Recall@5 across the four benchmarks. CAGE achieves the strongest results on 2WikiMultihopQA (85.4), outperforming Vanilla RAG by 5 points and ColBERTv2 by 17 points, and on HotpotQA (84.2), matching monoT5 and outperforming all other baselines; 
and MuSiQue (64.0), edging out Vanilla RAG (63.6) and monoT5 (62.8). 
The strongest gains appear on bridge-dominated datasets (2WikiMultihopQA: 48\% bridge; HotpotQA: 81\% bridge), precisely the setting where structural coherence identifies connecting passages that relevance-only scoring overlooks. On MuSiQue, where 3-4 hop chains increase distractor density, the margin is narrower but CAGE still achieves the top score, suggesting that coherence signals provide consistent benefit even when structural overlap between gold passages and distractors is high. On RAMDocs, all methods cluster within 1 point (86.8--87.6), indicating that adversarial misinformation, where contradictory passages share entity structure with gold ones, cannot be resolved through structural connectivity alone.

\begin{figure}[h!]
    \centering 
    \includegraphics[width=0.50\textwidth]{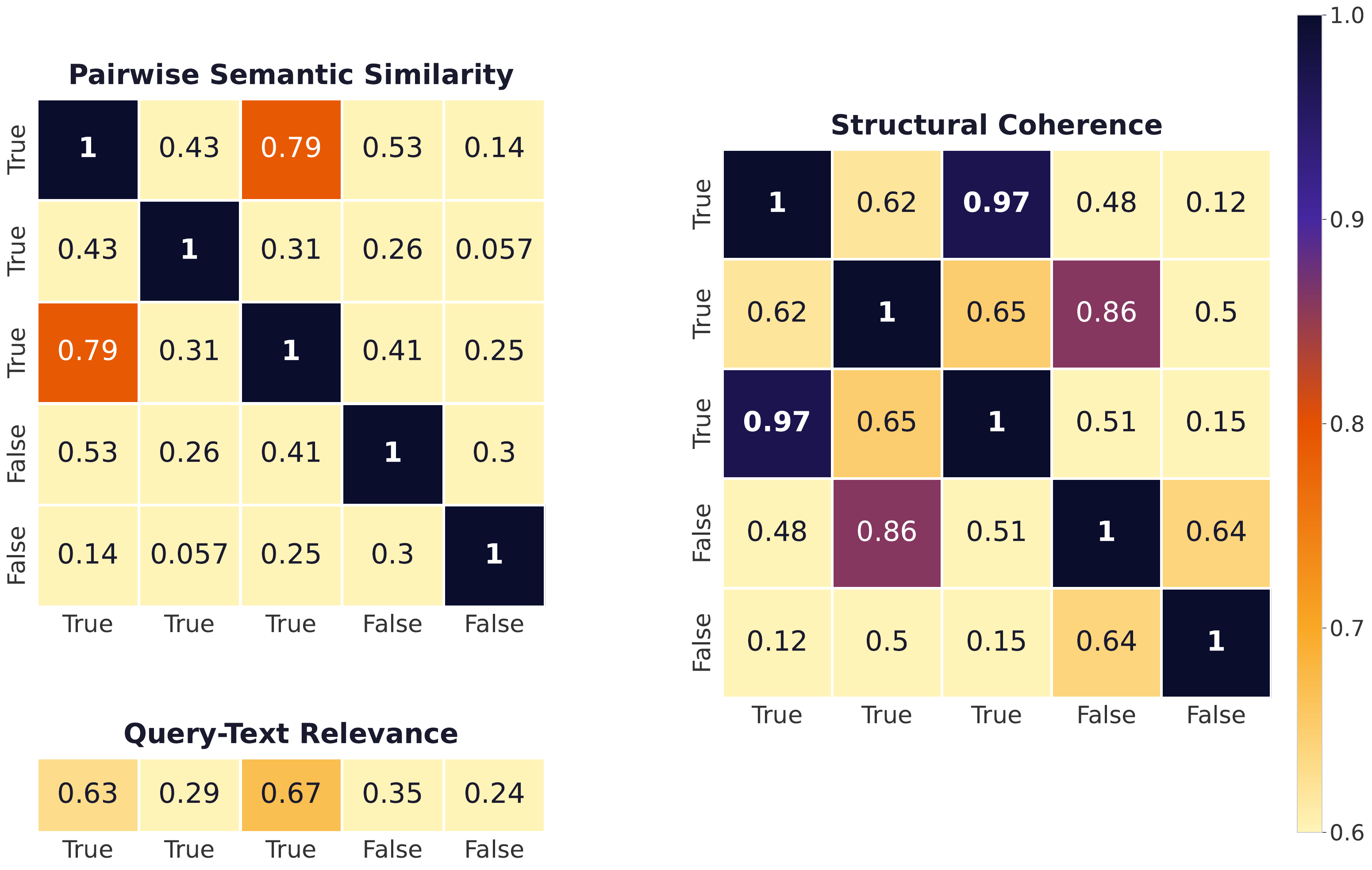}
    \caption{Semantic similarity (left) vs. structural coherence (right) for a retrieved set of five passages (three gold, two distractors). Semantic similarity and query relevance fail to unify the gold set, the second gold passage scores 0.29 against the query. Structural coherence recovers the latent connectivity among all three gold passages, forming a separable cluster that distinguishes supporting evidence from distractors.}
    \label{fig:heatmapexample1}
\end{figure}





\begin{table*}[h]
\centering
\small 
\setlength{\tabcolsep}{3pt} 
\begin{tabular}{lcccccccccccccccc}
\toprule
\multirow{2}{*}{\textbf{Model}} & \multicolumn{8}{c}{\textbf{Gemma 4 26B A4B}} & \multicolumn{8}{c}{\textbf{Llama 3.3 70B}} \\

  & \multicolumn{2}{c}{\textbf{MuSiQue}} & \multicolumn{2}{c}{\textbf{2Wiki}} & \multicolumn{2}{c}{\textbf{HotpotQA}} & \multicolumn{2}{c}{\textbf{RAMDocs}} &
 \multicolumn{2}{c}{\textbf{MuSiQue}} & \multicolumn{2}{c}{\textbf{2Wiki}} & \multicolumn{2}{c}{\textbf{HotpotQA}} & \multicolumn{2}{c}{\textbf{RAMDocs}}
 \\ 

 \cmidrule(lr){2-3} \cmidrule(lr){4-5} \cmidrule(lr){6-7} \cmidrule(lr){8-9} \cmidrule(lr){10-11} \cmidrule(lr){12-13} \cmidrule(lr){14-15} \cmidrule(lr){16-17}
 \textbf{Retriever} & EM & F1 & EM & F1 & EM & F1 & EM & F1 & EM & F1 & EM & F1 & EM & F1 & EM & F1 \\ \midrule
monoT5 & 15.8 & 17.0& 34.2 & 34.4 & 50.3 & \underline{51.3} & 42.7 & \textbf{59.8} & 24.6 & 28.3 & 47.6 & 49.1 &\underline{54.2} &\textbf{59.5} & 43.3& \textbf{67.6}\\
Vanilla RAG &\underline{16.6} & \underline{17.6} & \underline{38.1} & \underline{38.8}& \textbf{51.7}& \textbf{51.9} & \textbf{45.3} & \underline{59.7}
& \underline{29.0} & \underline{32.2}&\underline{52.6}  & \underline{56.0} &50.6 & \underline{58.7} & \textbf{45.7} & \underline{66.1}
\\
CAGE & \textbf{17.1} & \textbf{18.0} &  \textbf{40.1}&  \textbf{40.2} & \underline{50.7} & 50.8 & \underline{44.7} & 59.3 & \textbf{30.4} & \textbf{33.1} & \textbf{54.3} & \textbf{56.9}  & \textbf{55.0} & 58.4 &  \underline{44.7}&65.8 \\
\bottomrule
\end{tabular}
\caption{Generation Performance Comparison on MuSiQue, 2Wiki, HotpotQA and RAMDocs Datasets}
\label{tab:generation_1}
\end{table*}

\subsection{Answer Generation}
Beyond retrieval quality, we evaluate the downstream impact of our reranking strategies on answer generation. We provide the top-5 retrieved chunks as context to two popular language models: \textit{Gemma 4 26B A4B} \cite{google2026gemma4} and \textit{Llama 3.3 70B} \cite{meta2024llama33} with temperature 0. We report Exact Match (EM) and F1, comparing CAGE against monoT5 and Vanilla RAG which are the two strong baselines on Recall@5 (Table \ref{tab:recalls_r5}).

Table \ref{tab:generation_1} presents end-to-end results. CAGE achieves the highest EM on 5 of 8 dataset–model combinations, while F1 improvements are more modest. This performance asymmetry highlights a core benefit of coherence-aware reranking. EM is a strict, fine-grained metric requiring the complete reasoning chain to be present in context, a single missing bridge passage drops the score to zero. Conversely, F1 rewards partial token overlap and remains tolerant of incomplete or noisy context. Because CAGE specifically rescues structurally necessary bridge passages that relevance-only scoring demotes, the resulting EM gains scale directly with dataset bridge density, yielding the most pronounced improvements on bridge-dominated benchmarks. This EM advantage demonstrates that promoting set-level structural coherence supplies the generator with the complete evidence chain necessary for exact multi-hop answer extraction, rather than merely superficial text overlap.


\begin{figure}[h!]
    \centering 
    \includegraphics[width=0.48\textwidth]{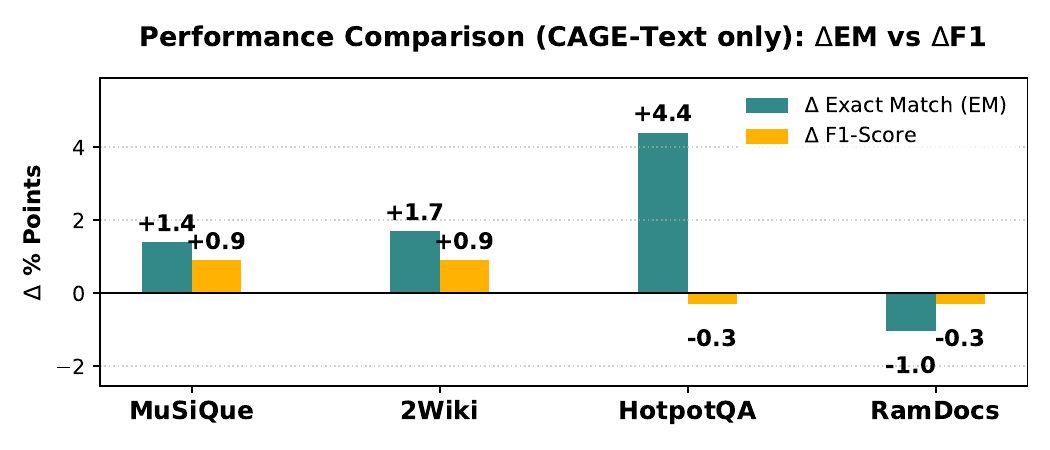}
    \caption{Ablation of the coherence signal. Bars show the difference in EM and F1 (percentage points) between CAGE's fused score and text-only query relevance (Vanilla RAG) based on Llama model, isolating the contribution of the graph-based coherence term $s_i$.}
    \label{fig:emf1}
\end{figure}


\subsection{Ablation: Coherence Signal Contribution} 
To isolate the architectural contribution of CAGE, Table \ref{tab:ablation1} presents a component-wise ablation of Recall@5 across all four benchmarks. Replacing the standard GCN encoder with an R-GCN encoder(CAGE No-reweight) yields the single largest retrieval improvement, most notably on MuSiQue, where Recall@5 jumps by +10.5 points. This confirms that relation-type-specific transformation matrices are essential for capturing heterogeneous multi-hop connectivity. Incorporating min-out-degree node reweighting (CAGE Full) provides an additional performance boost across all datasets, achieving peak recall on 2Wiki (+3.9 points) and HotpotQA (+1.8 points). The corresponding EM and F1 results are listed in Table~\ref{tab:ablation2} in Appendix~\ref{appd:table}. This validates our structural prior that scaling low out-degree nodes effectively amplifies grounded factual anchors during representation learning.

Beyond architectural components, we isolate the effect of graph-based coherence, Figure~\ref{fig:emf1} compares the full fused score $f_i = \alpha \cdot \text{sim}(q, C_i) + (1-\alpha)\frac{1}{k-1} s_i$ against text-only ranking using $\text{sim}(q, C_i)$ alone (in Vanilla RAG). The coherence term $s_i$ drives consistent EM gains on MuSiQue (+1.4), 2Wiki (+1.7), and HotpotQA (+4.4), while F1 improvements remain modest (+0.9, +0.9, -0.3). 
This asymmetry confirms that coherence reranking primarily enhances extraction precision: EM requires the context to contain the complete, unfragmented reasoning chain, whereas F1 tolerates partial token overlap from noisy evidence. The largest gain (+4.4 EM on HotpotQA) aligns with its high proportion of bridge questions (81\%), where the coherence score successfully rescues structurally necessary passages. 

The slight decline on RAMDocs (-1.0 EM) reflects a limitation: adversarial misinformation passages share surrounding entity structure with gold facts, rendering structural connectivity alone insufficient to resolve logical conflicts. These differences align with dataset characteristics (Table \ref{tab:datasets_summary}): HotpotQA and 2Wiki reward structural connectivity between passages, where the coherence term excels. MuSiQue's 3–4 hop chains with high distractor density mean the coherence gain is real but smaller, as distractors share entity structure with gold passages. RAMDocs is adversarial by design, contradictory passages are topologically similar to gold ones, limiting structural signals.

To understand this failure mode, we examine RAMDocs cases where CAGE demotes gold passages. Adversarial passages are constructed by modifying only the target fact while preserving all surrounding entity structure, producing near-isomorphic graphs that share the same nodes and relation types but differ in a single edge target (see Figure \ref{fig:dim4example}). Because the R-GCN assigns high structural coherence scores to both variants, this confirms that topology-based encoding cannot distinguish assertion-level factual shifts, indicating that complementary claim-level NLI verification is required for adversarial settings.

\begin{table}[!ht]
\centering
\small
\resizebox{\columnwidth}{!}{%
\begin{tabular}{lccccc}
\toprule
\textbf{Retriever} & \textbf{MuSiQue} & \textbf{2Wiki} & \textbf{HotpotQA} & \textbf{RAMDocs}  \\
\midrule

GCN encoder & 52.1 & 79.7& 80.3 & 86.3\\
CAGE (No-reweight)  & 62.6       & 81.5 & 82.4        & 86.7 \\
CAGE (full)              & 64.0        & 85.4 & 84.2         & 87.0 \\

\bottomrule
\end{tabular}%
}
\caption{Ablation study evaluating Recall@5 across all benchmarks to isolate the impact of the R-GCN architecture (vs. a vanilla GCN) and the min-out-degree node reweighting mechanism.}
\label{tab:ablation1}
\end{table}


\section{Conclusion}

We introduced CAGE, a framework that reranks retrieved passages by structural coherence rather than individual query relevance alone. Our central contribution is a four-dimensional decomposition of between-chunk coherence: Intra-domain relevance, noise resistance, informational bonding, and factual consistency, operationalized through directed heterogeneous entity graphs encoded via R-GCN. Empirically, CAGE achieves the highest Recall@5 on bridge-dominated benchmarks (2WikiMultihopQA, HotpotQA) and consistently improves Exact Match in downstream generation, confirming that structurally coherent context translates to more precise answers. Notably, even when retrieval recall decreases (MuSiQue), generation quality improves, demonstrating that between-chunk coherence is a complementary signal to per-chunk relevance, not a substitute for it.

\section*{Limitations}
The current framework has three primary constraints. First, graph construction relies on rule-based dependency parsing (spaCy), which limits relation extraction quality to the parser's accuracy; errors in entity recognition or dependency resolution propagate into the graph topology. Second, no existing RAG benchmark provides fine-grained between-chunk coherence annotations along our four dimensions; our evaluation relies on multi-hop supporting-fact annotations as an empirical proxy for between-chunk coherence rather than direct, per-dimension labels. Constructing a dedicated benchmark with explicit per-dimension coherence labels, a direction made tractable by the theoretical formalization we introduce in this work remains an important objective for future research.
Third, while the tunable fusion weight $\alpha$ balances coherence and relevance signals, it is currently set via grid search; an adaptive mechanism that adjusts $\alpha$ per query based on retrieval confidence or graph density could improve robustness on adversarial benchmarks like RAMDocs.

Our evaluation relies on multi-hop supporting-fact annotations as a proxy for the four coherence dimensions rather
  than direct, per-dimension coherence labels, since no existing RAG benchmark provides such fine-grained
  annotations. Constructing a benchmark with explicit per-dimension coherence labels, made tractable by the
  formalization we introduce in this work, is an important direction for future work.

\section*{Ethical considerations}                                                                    
This work uses publicly available multi-hop question answering benchmarks (HotpotQA, 2WikiMultihopQA, MuSiQue) derived from Wikipedia, which do not contain personally identifying information or offensive content. No new data involving human subjects was collected, and no crowd-sourcing or annotation was performed.                                                

Our answer generation relies on open-weight models (Gemma and LLaMA), which are run locally without transmitting data to external APIs. The CAGE reranking pipeline itself is lightweight (approximately 210K parameters) and requires minimal computational resources, making it accessible to researchers with limited hardware.

We do not foresee direct negative societal impacts from this work. However, as with all retrieval-augmented generation systems, improved retrieval coherence does not guarantee factual correctness of generated answers. Users of RAG systems should exercise caution and verify outputs in high-stakes domains such as healthcare or legal applications.



\newpage

\newpage

\bibliography{custom}

\clearpage 
\onecolumn 

\appendix

\section{Dataset summary}
\label{app:data}
The dataset information is listed in Table \ref{tab:datasets_summary}
\begin{table*}[h]
\centering
\small
\begin{tabular}{@{}c l c l c @{}}
\toprule
\textbf{Dataset} & \textbf{Samples} & \textbf{ \# candidates} & \textbf{Type / Characteristics} \\ 
\midrule
\textbf{MuSiQue} & 1,000 & 20 & Logical Chains, High Distractor Density \\
\textbf{HotpotQA} & 1,000 & 10 & Bridge (81\%), Comparison (19\%) \\
\textbf{2Wiki} & 1,000 & 10 & Comp. (41\%), Bridge (48\%), Inf. (11\%) \\
\textbf{RAMDocs} & 500 & [3,12]&  Adversarial Misinformation  \\
\bottomrule
\addlinespace
\end{tabular}
\caption{Summary of evaluation datasets, including sample sizes and linguistic characteristics across four dimensions.}
\label{tab:datasets_summary}
\end{table*}

\section{Proof of Relational Coherence Theorem }
\label{sec:appx_proof}
\textbf{Notations:} Our framework utilizes a structured notation to formalize the relational graph operations. We define a relational graph as $G = (V, E, \mathcal{R})$, where $V$ represents the set of nodes, $E$ the edges, and $\mathcal{R}$ the set of unique relations. For each node $i$ at a specific layer $l$, the hidden state is denoted by $h_i^{(l)}$. The transformation of these states is governed by $\bW_r^{(l)}$, which represents the relation-specific weight matrix for a given relation $r$ at layer $l$. The high-level representation of the graph is produced by the global graph encoding function $\Phi : \mathcal{G} \to \mathcal{E}$, which maps the graph structure into a latent embedding space $\mathcal{E}$. Finally, we evaluate the quality of these embeddings through the unified coherence manifold, which is characterized by its sub-dimensions $\mathcal{C}_{rel, noise, bond, cons}$ where $\mathcal{C}$ represents the Coherence Manifold which is a theoretical scalar representing the structural integrity of a document relative to a query.

\subsection{Proof of Proposition \ref{thm1}}
To prove that an R-GCN is as discriminative as the Relational Weisfeiler-Lehman (R-WL) algorithm, we must demonstrate that its aggregation function is injective over relational multisets.

\paragraph{Proof:}
Let the relational multiset of node $i$ at layer $l$ be defined as $\mathcal{M}_i = \{\!\{(r, h_j^{(l)}) : (j,r,i) \in E\}\!\}$. Consider the unnormalized update
  \[
  h_i^{(l+1)} = \sigma \left( \sum_{r \in \mathcal{R}} \bW_r^{(l)} \sum_{j \in \mathcal{N}_i^r} h_j^{(l)} + \bW_0^{(l)} h_i^{(l)} \right).
  \]
As established by \citet{xu2018powerful}, a GNN aggregator is as powerful as the WL test if and only if its neighbor-aggregation function is injective over multisets; sum aggregation satisfies this, whereas mean or max aggregation does not, since both can collapse multisets that share the same support but differ in multiplicity. By assigning a unique weight matrix $\bW_r$ to each $r \in \mathcal{R}$, the model partitions the feature space into $|\mathcal{R}|$ independent linear channels; the sum of injective functions over disjoint partitions remains injective, so the R-GCN can distinguish any two relational multisets $\mathcal{M}_1 \neq \mathcal{M}_2$, provided the matrices $\bW_r$ have sufficient rank. Thus $h_u^{(L)} = h_v^{(L)}$ if and only if the $L$-hop relational neighborhoods of $u$ and $v$ are isomorphic.

Our implementation instead uses the degree-normalized update with $\frac{1}{c_{i,r}}$ weighting -- a mean-style aggregator that is therefore not injective in general, by the same result of \citet{xu2018powerful}: two relational multisets whose per-relation neighbor sums differ only by the ratio of their degrees $c_{i,r}$ produce identical normalized aggregates. Injectivity is recovered when $c_{i,r}$ is constant across the compared neighborhoods, in which case normalization is a uniform rescaling of the sum case above. We adopt the normalized form for gradient stability during training and rely on the diversity training objective Section \ref{sec:4.3} to discourage the embedding collapse that the unnormalized case rules out analytically.

\subsection{Proof of Lemma 1}
\begin{lemma}{}
\label{lemma1}
Embedding similarity between two passage graphs is bounded by their node-level linguistic overlap.
\end{lemma} 

\paragraph{Proof:}
Let $x_i, x_j$ be input feature vectors for nodes in graphs $G_i$ and $G_j$ respectively, consisting of contextual embeddings, one-hot POS, and entity type indicators. At initialization, $h^{(0)} = \bW_{\text{proj}} x$. The graph-level similarity after sum pooling is:
  \[
  S(\Phi(G_i), \Phi(G_j)) = \cos\left(\sum_{u \in V_i} h_u^{(L)},\; \sum_{v \in V_j} h_v^{(L)}\right)
  \]
If $V_i$ and $V_j$ share no linguistic overlap (no common entities, POS patterns, or GloVe-similar terms), their initial projections lie in near-orthogonal subspaces. Since the R-GCN update is a linear transformation followed by non-linearity, it cannot generate significant alignment from orthogonal initial states. Therefore, linguistic overlap is a necessary condition for high embedding similarity, establishing that $S$ naturally gates out passage pairs with no shared referential content (Intra-Domain Relevance).

Mathematically, orthogonal feature spaces prevent similarity aggregation. In natural language terms, two passages cannot exhibit intra-domain relevance if they share no referential anchors or domain entities, naturally filtering out unrelated topics prior to deep graph encoding.

\subsection{Proof of Lemma 2}

\begin{lemma}
\label{lemma2}
    The architecture attenuates signal from isolated informational ``islands.''
\end{lemma}
 
\paragraph{Proof:}
Consider a noise node $v_n$ that is semantically similar to the query but topologically isolated (degree 0 or connected only within a small disconnected component). In the R-GCN update, $v_n$ receives no neighbor messages (or messages only from its small island). Its updated embedding $h_{v_n}^{(1)}$ therefore reflects only its self-loop:
$\|h_{v_n}^{(L)}\|_2 =\| \sigma(\bW_0^{(L-1)} h_{v_n}^{(L-1)})\|_2$, lacking the relational enrichment that connected nodes receive. 
By contrast, consider a connected backbone node $v$ with out-degree $d = \sum_{r \in \mathcal{R}} |\mathcal{N}_v^r|$, unrolling the $L-$layer R-GCN recursion over its relational neighborhood yields:
\[
\|h_v^{(L)}\|_2 \geq d \cdot \min_r \sigma_{\min}(\bW_r)^{L-1} \cdot \|h_{avg}\|_2
\]
where $\sigma{\min}(\bW_r)$ denotes the minimum singular value of the relation matrix $\bW_r$ and $h_{avg}$ is the average neighbor feature vector. Backbone node representations strictly dominate isolated node representations in the global sum pool $\Phi(G)$ when
\[
d \cdot \min_r \sigma_{\min}(\bW_r) > 1
\]
which is guaranteed by the sufficient-rank condition established in Proposition 1. The diversity training objective further separates graphs containing noise islands from structurally coherent graphs, as isolated nodes contribute uninformative signal that reduces structural similarity across co-retrieved passage graphs\cite{gasteiger2018predict}.

Cross-domain distractors (e.g., keyword-matched homonyms) act as topologically isolated nodes. Graph convolutions mathematically dampen these ungrounded islands while exponentially amplifying well-connected factual backbones, guaranteeing robust noise resistance.


\subsection{Proof of Lemma 3}
\begin{lemma}
\label{lemma3}
 Shared bridge entities elevate cross-graph similarity through relational enrichment, without requiring a cross-document edge.
\end{lemma}

\paragraph{Proof:}
Let $e$ be a bridge entity realized as a node in both $D_1$'s graph $G_1=(V_1,E_1)$ and $D_2$'s graph $G_2=(V_2,E_2)$ (i.e.\ the same canonical lemma/entity after coreference resolution, per Section~4.1). Within $G_1$, $e$ participates in relational edges $(r, e, \cdot)$ reflecting its role in $D_1$'s factual assertions; by the R-GCN update, its layer-1 embedding
  \[
  h_e^{(1)} = \sigma\!\left(\sum_{r \in \mathcal{R}} \bW_r^{(1)} \sum_{j \in \mathcal{N}_e^r} \frac{1}{c_{e,r}} h_j^{(0)} + \bW_0^{(1)}
  h_e^{(0)}\right)
  \]
aggregates $D_1$'s local relational context around $e$, rather than reflecting only $e$'s raw input feature $x_e = h_e^{(0)}$. The analogous embedding computed inside $G_2$ aggregates $D_2$'s independent local context around the same node identity.

Because $\Phi(G_1) = \sum_{u \in V_1} h_u^{(1)}$ and $\Phi(G_2) = \sum_{v \in V_2} h_v^{(1)}$ are sums over all nodes in each graph, both
sums include a term anchored at $e$. By Lemma~\ref{lemma1}, the shared identity of $e$ guarantees $h_e^{(0)}$ is near-identical across $G_1$ and $G_2$ prior to relational transformation; since $\bW_r^{(1)}$ is applied identically regardless of which graph $e$ sits in, the two enriched terms remain non-orthogonal after aggregation. This produces a strictly positive contribution to $\cos(\Phi(G_1), \Phi(G_2))$ anchored specifically at $e$, one that does not depend on $D_1$ and $D_2$ sharing any other vocabulary.

By contrast, a passage pair sharing no entity node reduces to the generic overlap case of Lemma~\ref{lemma1} with no such anchored term, and a passage pair sharing only topologically isolated, unenriched nodes reduces to the noise case of Lemma~\ref{lemma2}. We therefore interpret informational bonding as an emergent property of shared entity nodes being independently relationally enriched within each passage's own 1-hop neighborhood, consistent with our per-passage graph construction and pairwise cosine scoring, bonding does not require, and our pipeline does not construct, an explicit edge crossing the document boundary.






\subsection{Proof of Lemma 4}
\begin{lemma}
\label{lemma4}
Relational contradictions that alter relation type or ordering along a shared reasoning path are penalized through non-commutative operators.
\end{lemma}

\paragraph{Proof:}
Standard graph models treat edges as scalars, making the logic $A \to B \to C$ identical to $C \to B \to A$. In an R-GCN, the relationship is a linear operator $\bW_r$.

Consider a factual path $P_1 = (r_1, r_2)$. The transformation is $\bW_{r_2} \bW_{r_1}$. If a retrieved document presents a contradiction $P_2 = (r_2, r_1)$, the transformation is $\bW_{r_1} \bW_{r_2}$. Since matrix multiplication is \textbf{non-commutative} ($\bW_{r_1} \bW_{r_2} \neq \bW_{r_2} \bW_{r_1}$), the resulting node embeddings $h^{(L)}$ will diverge in the manifold $\mathcal{E}$.

  \[
  \| \Phi(G_{consistent}) - \Phi(G_{contradictory}) \| > \epsilon
  \]

This divergence ensures that the similarity score $\mathcal{S}$ is maximized only when the relational logic of the document set aligns perfectly with the query's expected logical flow, thereby enforcing internal factual harmony \cite{schlichtkrull2018modeling}. Note that Lemma~\ref{lemma4} holds at convergence of a well-trained model, and we reference the contrastive objective as the mechanism that drives $\bW_r$ matrices to be relation-discriminative.

This argument applies to contradictions that manifest as a change in relation \emph{type or ordering} along a path $(r_1, r_2)$, e.g., a reversed causal or temporal relation. It does not cover contradictions expressed as a single-entity substitution within an otherwise \emph{identical} relational structure (e.g., replacing only a location argument while every surrounding relation type and edge direction is preserved): such an edit changes only the leaf node $h_{j_p}^{(0)}$ at the end of a path, not the operator sequence $\prod_k \bW_{r_k}$, so the resulting divergence is bounded by the difference between the two leaf embeddings rather than by non-commutativity, and can be small when the substituted entities are lexically or semantically similar. This is precisely the assertion-level factual shift underlying our RAMDocs failure analysis (Section~6.3): we attribute that result to this boundary condition of Lemma~\ref{lemma4} rather than to a general architectural failure, motivating the complementary claim-level (e.g., NLI) verification we identify as future work.

Relation matrices act as direction and type sensitive logical operators. Just as swapping subject and object alters factual meaning in natural language, matrix non-commutativity ensures that reversed or incompatible relation types map to divergent embeddings, penalizing that class of factual inconsistency.

\section{Prompts}
\label{app:prompt}

\begin{figure*}[h]
    \centering
    \includegraphics[width=0.8\linewidth]{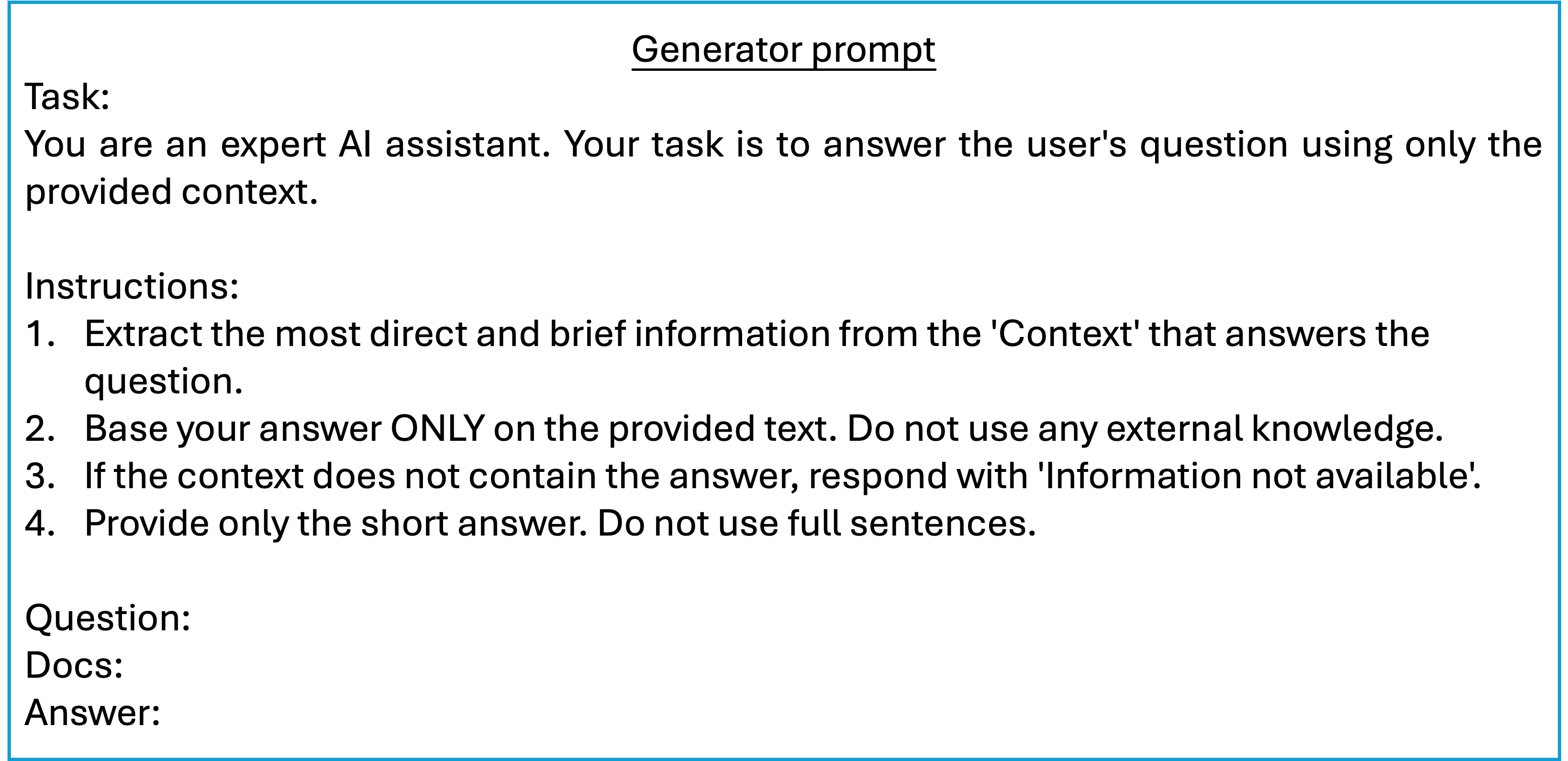}
    \caption{Prompt for Generator}
    \label{fig:generator}
\end{figure*}

\section{Figures}
\label{app:figs}

Figure \ref{fig:alpha} illustrates the sensitivity of Recall@5 to the fusion weight $\alpha$ across all four evaluation benchmarks.

\begin{figure*}[h]
    \centering
    \includegraphics[width=0.8\textwidth]{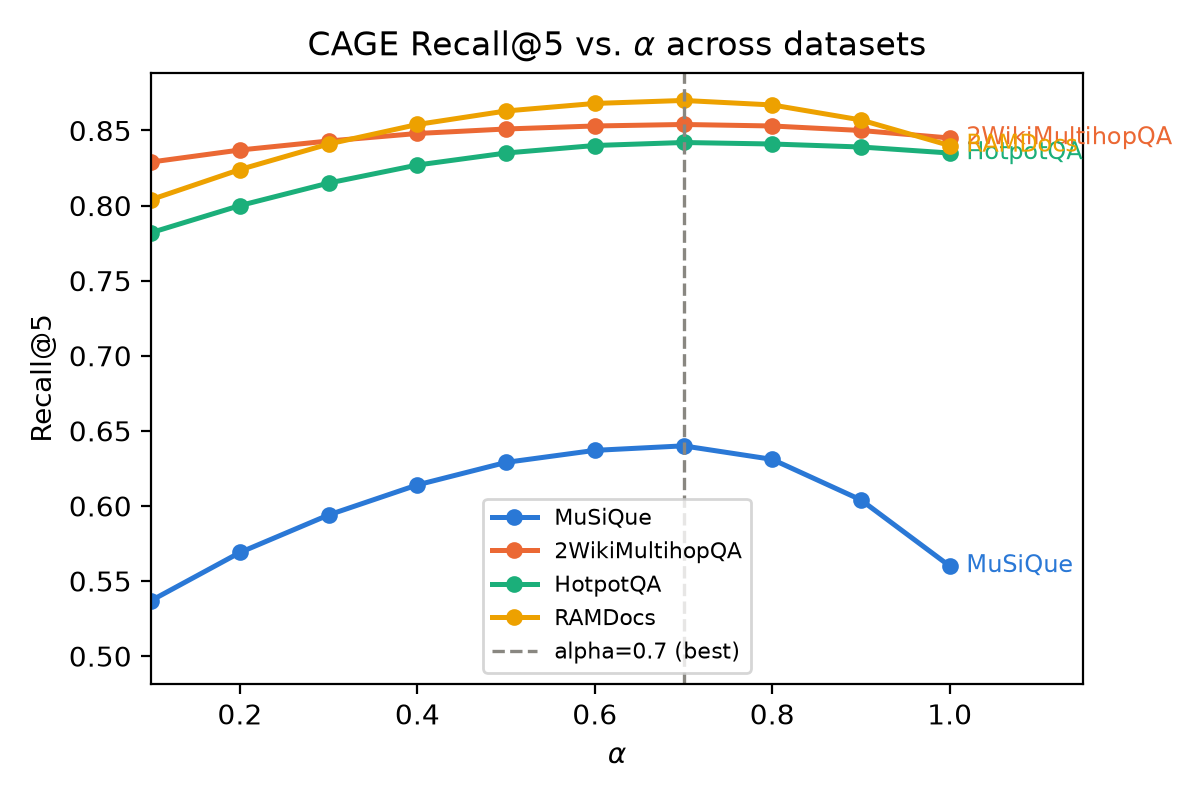}
    \caption{Sensitivity of Recall@5 to the fusion hyperparameter $\alpha$ across four benchmarks.}
    \label{fig:alpha}
\end{figure*}

\section{Examples for Each Dimension}
\label{app:example}

To provide qualitative intuition for our theoretical framework, this section details case studies mapping each benchmark dataset to its corresponding between-chunk coherence dimension (Figures \ref{fig:dim1example}–\ref{fig:dim4example}).

\begin{figure*}[h]
    \centering
    \includegraphics[width=0.8\textwidth]{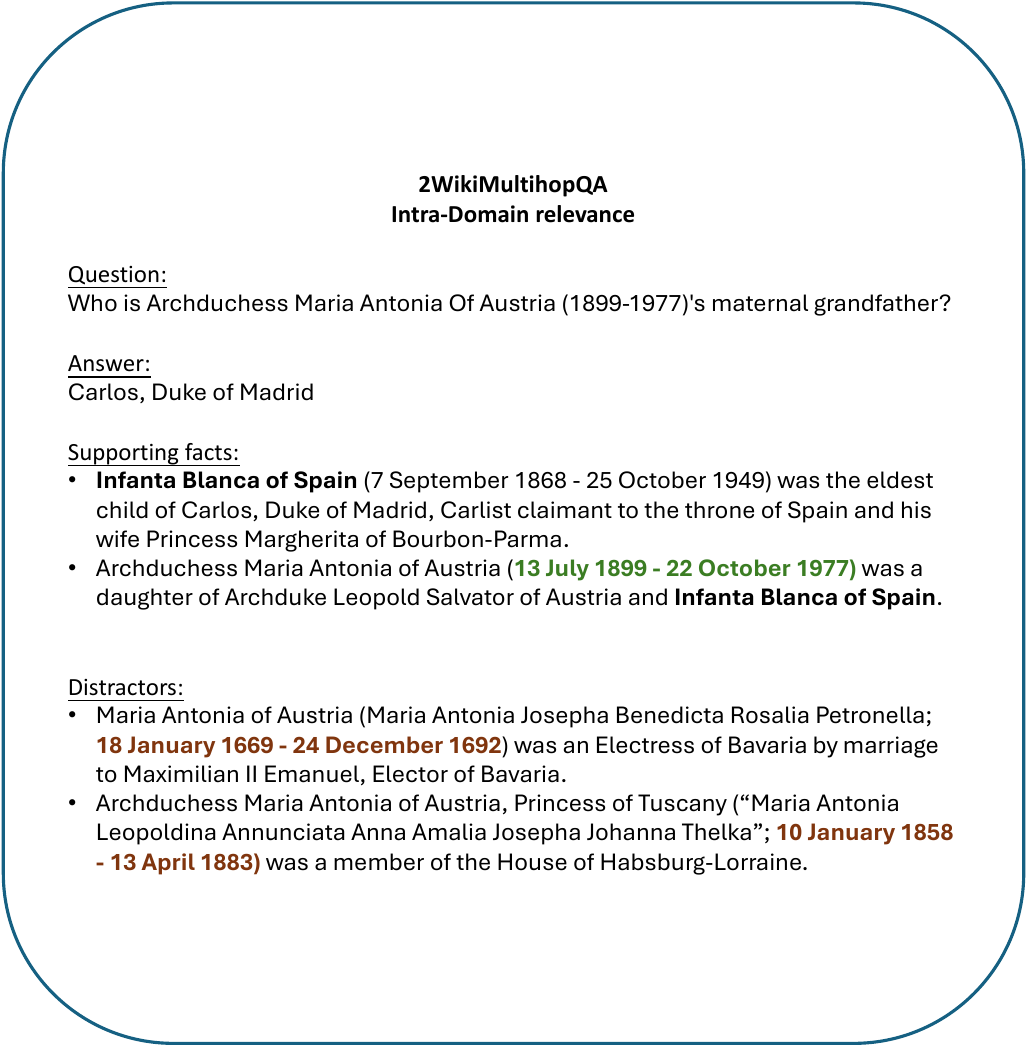}
    \caption{Case Study of Intra-Domain Relevance. \\
    This example illustrates the challenge of entity disambiguation. The query seeks the "maternal grandfather" of a specific Maria Antonia of Austria (1899–1977). While distractors share the identical name, they refer to different individuals from the 17th and 19th centuries. Furthermore, the term "grandfather" does not appear in the supporting facts; instead, the text only mentions her mother. 
    }
    \label{fig:dim1example}
\end{figure*}

\begin{figure*}[h]
    \centering
    \includegraphics[width=0.8\textwidth]{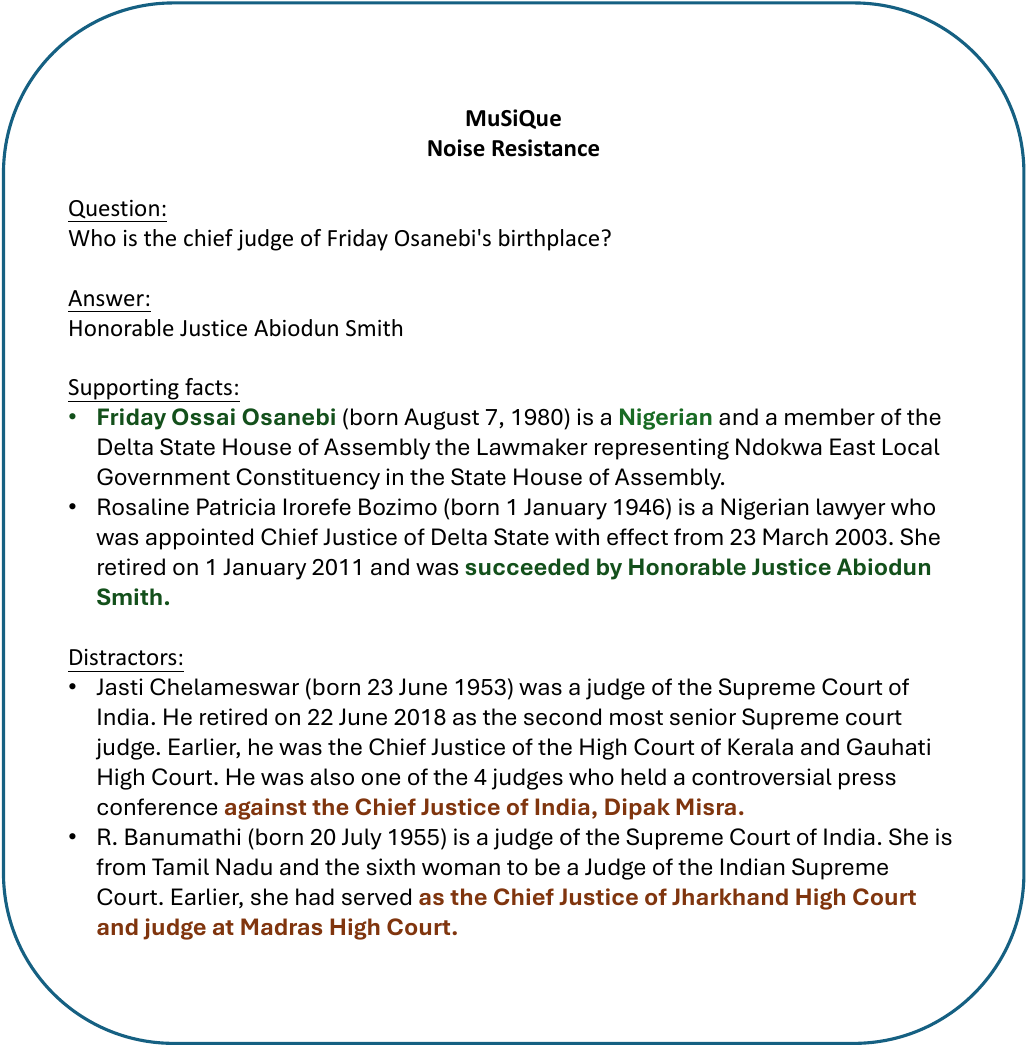}
    \caption{Case Study of Noise Resistance. \\
    This MuSiQue example demonstrates the challenge of contextual noise filtering. The query asks for the chief judge of Friday Osanebi's birthplace. While the distractors contain highly relevant professional terms, "Chief Justice," "Supreme Court," and "High Court"—they refer to the legal systems of India (Kerala, Jharkhand).
    }
    \label{fig:dim2example}
\end{figure*}

\begin{figure*}[h]
    \centering
    \includegraphics[width=0.8\textwidth]{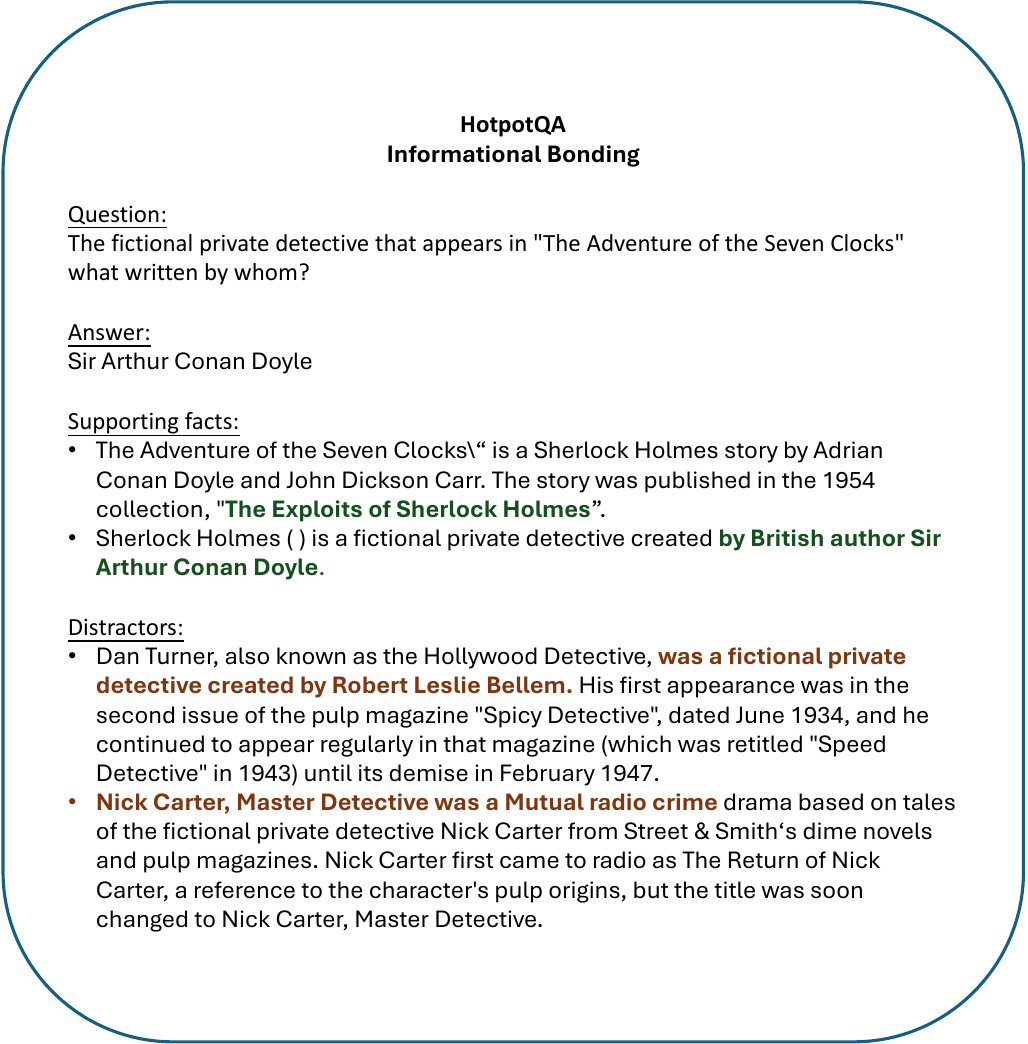}
    \caption{Case Study of Informational Bonding. \\
    This HotpotQA sample illustrates a multi-hop reasoning "bridge." The query asks for the creator of a detective who appears in a specific story (The Adventure of the Seven Clocks). To answer the question, it needs to first identify that the detective in the story is Sherlock Holmes, and then "bridging" to the second document to find that Holmes was created by Sir Arthur Conan Doyle.
    }
    \label{fig:dim3example}
\end{figure*}

\begin{figure*}[h]
    \centering
    \includegraphics[width=0.8\textwidth]{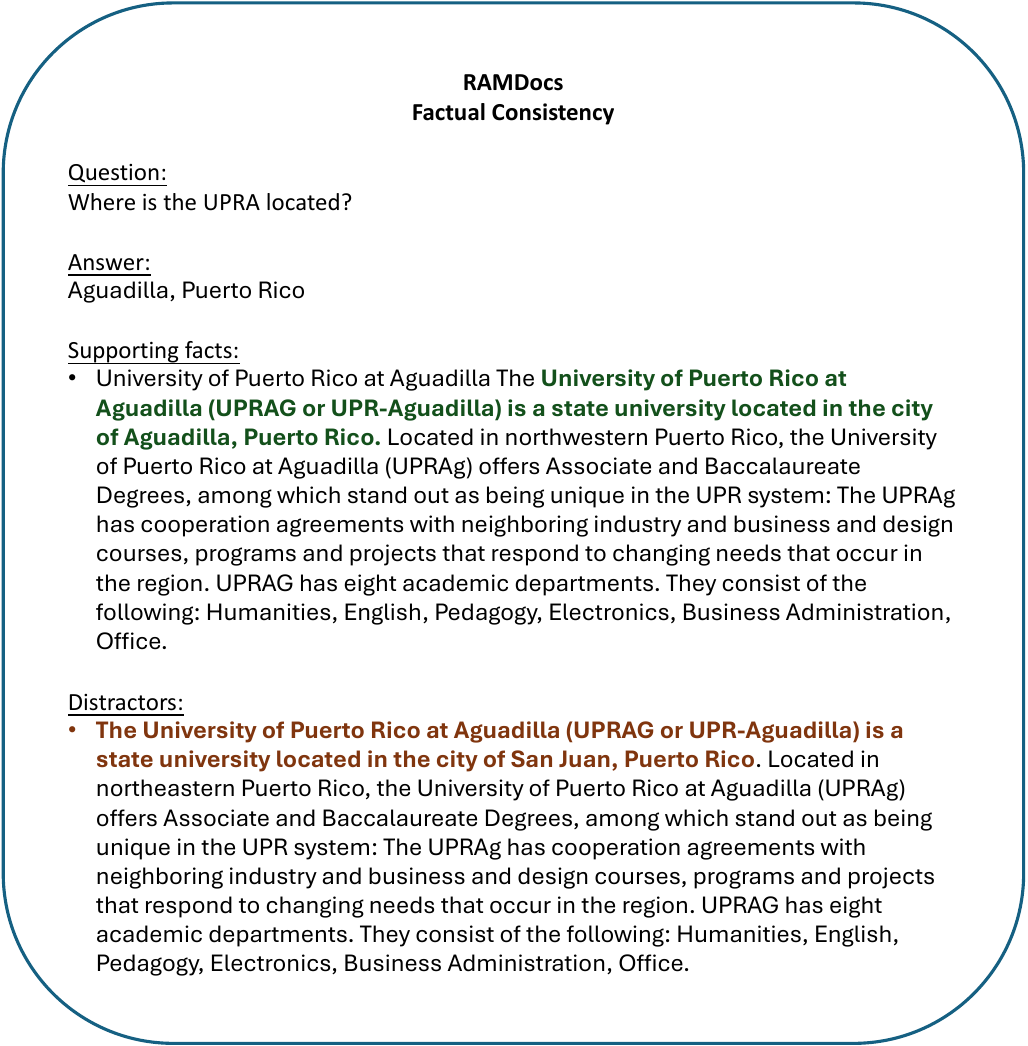}
    \caption{Case Study of Factual Consistency. \\
    This sample presents an adversarial conflict where two retrieved chunks provide contradictory locations for the "UPRA" university. Both documents share identical phrasing regarding departments and degrees, but the Supporting Fact places the campus in Aguadilla (Northwestern PR), while the Distractor (misinformation) places it in San Juan (Northeastern PR).
    }
    \label{fig:dim4example}
\end{figure*}

\section{Table for Ablation Study}
\label{appd:table}
\begin{table*}[h]
\centering
\small 
\setlength{\tabcolsep}{3pt} 
\begin{tabular}{lcccccccccccccccc}
\toprule
\multirow{2}{*}{\textbf{Model}} & \multicolumn{8}{c}{\textbf{Gemma 4 26B A4B}} & \multicolumn{8}{c}{\textbf{Llama 3.3 70B}} \\

  & \multicolumn{2}{c}{\textbf{MuSiQue}} & \multicolumn{2}{c}{\textbf{2Wiki}} & \multicolumn{2}{c}{\textbf{HotpotQA}} & \multicolumn{2}{c}{\textbf{RAMDocs}} &
 \multicolumn{2}{c}{\textbf{MuSiQue}} & \multicolumn{2}{c}{\textbf{2Wiki}} & \multicolumn{2}{c}{\textbf{HotpotQA}} & \multicolumn{2}{c}{\textbf{RAMDocs}}
 \\ 

 \cmidrule(lr){2-3} \cmidrule(lr){4-5} \cmidrule(lr){6-7} \cmidrule(lr){8-9} \cmidrule(lr){10-11} \cmidrule(lr){12-13} \cmidrule(lr){14-15} \cmidrule(lr){16-17}
 \textbf{Retriever} & EM & F1 & EM & F1 & EM & F1 & EM & F1 & EM & F1 & EM & F1 & EM & F1 & EM & F1 \\ \midrule
GCN encoder & 9.5 & 10.6  & 33.3 & 33.6& 42.4 & 42.8 & 43.2 &58.8 & 21.5 & 24.1 & 45.4 & 47.7 & 49.9 & 53.1 & 42.2 & 65.3\\
CAGE(no reweighting) & 15.5 & 16.6& 37.1 & 38.1 & 50.1 & 50.3 & 43.3 & 58.9 & 28.4 & 31.9 & 52.5 & 50.7 &54.3 &57.4 & 44.2& 65.5\\
CAGE & \underline{17.1} & \underline{18.0} &  \underline{40.1}& \underline{40.2} & \underline{50.7} & \underline{50.8} & \underline{44.7} & \underline{59.3} & \underline{30.4} & \underline{33.1} & \underline{54.3} & \underline{56.9}  & \underline{55.0} & \underline{58.4} &  \underline{44.7}&\underline{65.8} \\


\bottomrule
\end{tabular}
\caption{Downstream generation ablation study across Gemma 4 26B A4B and Llama 3.3 70B generators. Evaluates Exact Match (EM) and F1 metrics across all four multi-hop datasets to isolate component contributions.}
\label{tab:ablation2}
\end{table*}

\section{License}
All datasets and models used in this work are employed in accordance with their intended research use.
\paragraph{Datasets.}
  HotpotQA \citep{yang2018hotpotqa} is released under CC BY-SA 4.0.
  2WikiMultihopQA \citep{ho2020constructing} is released under Apache 2.0.
  MuSiQue \citep{trivedi2022musique} is released under CC BY 4.0.
  RAMDocs \citep{wang2025retrieval} is released under MIT.

\paragraph{Models and Libraries.}
  all-MiniLM-L6-v2 \citep{reimers2019sentence} is released under Apache 2.0.
  SpaCy and \texttt{en\_core\_web\_lg} are released under MIT.
  GTR \citep{ni2022large} is released under Apache 2.0.
  ColBERTv2 \citep{santhanam2022colbertv2} is released under MIT.
  monoT5 \citep{nogueira2020document} is released under Apache 2.0.

  All artifacts are used for research purposes consistent with their respective licenses.

\section{Computational Budget}                                                                               
All experiments were conducted on a compute node equipped with 8 NVIDIA A100 GPUs. The CAGE pipeline's primary computational cost is graph construction: SpaCy \texttt{en\_core\_web\_lg} with coreference
  resolution processes approximately 10,000 passages across all four benchmarks in roughly 6 hours on CPU. R-GCN training is
   performed per query on $k{=}5$ passage graphs for 50 epochs (Adam, lr=$1 \times 10^{-3}$), with each instance completing
  in under 0.5 seconds. Total R-GCN training across all queries requires approximately 0.5 GPU-hours on a
  single A100. Including passage encoding and all baseline model inference (monoT5, ColBERTv2, GTR), the total
  computational budget is approximately 12 GPU-hours. The 8-GPU setup was used to parallelize experiments across datasets
  and baselines simultaneously; the CAGE pipeline itself requires only a single GPU due to its lightweight architecture.

\end{document}